\documentclass{article}

\PassOptionsToPackage{numbers, compress}{natbib}
\usepackage[final]{nips_2017}
\usepackage{graphicx}

\usepackage[utf8]{inputenc} % allow utf-8 input
\usepackage[T1]{fontenc}    % use 8-bit T1 fonts
\usepackage{hyperref}       % hyperlinks
\usepackage{url}            % simple URL typesetting
\usepackage{booktabs}       % professional-quality tables
\usepackage{amsfonts}       % blackboard math symbols
\usepackage{nicefrac}       % compact symbols for 1/2, etc.
\usepackage{microtype}      % microtypography
\usepackage{graphicx} % more modern
\usepackage{epsfig} % less modern
\usepackage{amsmath, amsthm, amssymb}
\usepackage{algorithm}
\usepackage{algorithmic}
\usepackage{multicol}

\newtheorem{theorem}{Theorem}

\newtheorem{proposition}[theorem]{Proposition}
\newtheorem{corollary}[theorem]{Corollary}
\newtheorem{definition}[theorem]{Definition}

\title{Unsupervised Submodular Rank Aggregation on Score-based Permutations}

\author{
  Jun Qi 	 \\
  Electrical Engineering \\
  University of Washington\\
  Seattle, WA 98105 \\
  \And
  Xu Liu			\\
  Institute of Industrial Science	\\
  The University of Tokyo		\\
  Japan, 153-8505	\\
   \And
   Javier Tejedor 	\\
   University San Pablo CEU	\\
  Madrid, Spain 	\\
  \And
  Shunsuke Kamijo	\\
  Institute of Industrial Science	\\
  The University of Tokyo	\\
  Japan, 153-8505	\\
}

\begin{document}

\maketitle

\begin{abstract}
Unsupervised rank aggregation on score-based permutations, which is widely used in many applications, has not been deeply explored yet. This work studies the use of submodular optimization for rank aggregation on score-based permutations in an unsupervised way. Specifically, we propose an unsupervised approach based on the Lovasz Bregman divergence for setting up linear structured convex and nested structured concave objective functions. In addition, stochastic optimization methods are applied in the training process and efficient algorithms for inference can be guaranteed. The experimental results from Information Retrieval, Combining Distributed Neural Networks, Influencers in Social Networks, and Distributed Automatic Speech Recognition tasks demonstrate the effectiveness of the proposed methods. 
\end{abstract}

\section{Introduction}

Unsupervised rank aggregation is the task of combining multiple permutations on the same set of candidates into one ranking list with a better permutation on candidates. Specifically, there are $K$ different permutations $\{\textbf{x}_{1}, \textbf{x}_{2}, ..., \textbf{x}_{K}\}$ in total, where the elements $\{X_{1j}, X_{2j}, ..., X_{Nj}\}$ in the ranking list $\textbf{x}_{j}$ denote either relative orders or numeric values for the candidates. Most of the previous work on unsupervised rank aggregation focuses on social choice theory \cite{RankWeb}, where relative orders for candidates are assigned to the elements in a ranking list and a consensus among all the ranking lists is the result of pursuit. Numerous aggregated methods, such as the unsupervised methods based on the local Kemeny optimality \cite{voting} or distance-based Mallows models \cite{Klementiev08}, aim to find a consensus in order to minimize the sum of distances between the permutations and the consensus. 

However, in the framework of rank aggregation on score-based permutations, the elements in the ranking lists represent numerical values and a combined list $\hat{\textbf{x}}$ with the aggregated scores is used to obtain the relative orders for candidates by sorting the values in $\hat{\textbf{x}}$. Thus, the methods used in the rank aggregation on order-based permutations are not simply generalized to the unsupervised rank aggregation on score-based permutations. To the best of our knowledge, only the Borda count-based unsupervised learning algorithm for rank aggregation (ULARA) \cite{Klementiev08} is particularly designed for the unsupervised rank aggregation on score-based permutations, although some naive methods like averaging and majority vote can also be applied. 

The Lovasz Bregman (LB) divergence for rank aggregation on the score-based permutations is initially introduced in \cite{rishabh_lbd}. The LB divergence is derived from the generalized Bregman divergence parameterized by the Lovasz extension of a submodular function, where a submodular function can be defined via the diminishing return property \cite{edmonds}. Specifically, a function $f: 2^{V} \rightarrow R_{+}$ is said to be submodular if for any item $a\in V\backslash B$ and subsets $A\subseteq B \subseteq V$, $f$ satisfies the inequality $f(\{a\}\cup A) - f(A) \ge f(\{a\}\cup B) - f(B)$. Many discrete optimization problems involving submodular functions can be solved in polynominal time, e.g., exact minimization or approximate maximization \cite{fujishige}. In addition, the Lovasz extension of a submodular function ensures a convex function \cite{Lovasz}. 

Although the introduction of the LB divergence to rank aggregation on score-based permutations is briefly discussed in \cite{rishabh_lbd},  the existing formulation of the LB divergence for rank aggregation and the related algorithms are limited to a supervised manner. Thus, the formulations of the LB divergence for the unsupervised rank aggregation on score-based permutations as well as the related algorithms are still lacking. Since the LB divergence is capable of measuring the divergence between a score-based permutation and an order-based permutation, this work applies the LB divergence to the unsupervised rank aggregation on the score-based permutations. In addition, efficient algorithms associated with the LB divergence framework are proposed accordingly. 

The significance of unsupervised rank aggregation on score-based permutations is that ground truths or relevance scores associated with candidates are required for the supervised cases, but they are quite expensive to obtain in practice. For example, in the field of information retrieval \cite{submod_bandits}, it is difficult to annotate the relevance scores for all the retrieval documents, which prevents the use of the learning to rank methods. Similarly, in the field of speech recognition \cite{jun16} or machine translation \cite{mt}, the ground truth is normally unknown and thus multiple $n$-best hypotheses, which are generated by distributed deep learning models, have to be combined in an unsupervised way. Thus, it is necessary to design theories and algorithms for the unsupervised rank aggregation on score-based permutations. 

\section{Preliminaries}
\subsection{The Definition of the LB Divergence}
The LB divergence as a utility function was firstly proposed by \cite{rishabh_bregman} and this can be used for measuring the divergence between a score-based permutation and an order-based permutation. The related definitions for the LB divergence are briefly summarized as follows:

\begin{definition}
For a submodular function $f$ and an order-based permutation $\sigma$, then there is a chain of sets $\phi = S_{0}^{\sigma} \subseteq S_{1}^{\sigma} \subseteq ... \subseteq S_{N}^{\sigma} = V$, with which the vector $h_{\sigma}^{f}$ is defined as (\ref{eq:hs}).
\begin{equation}
\label{eq:hs}
h_{\sigma}^{f}(\sigma(i)) = f(S^{\sigma}_{i}) - f(S^{\sigma}_{i-1}), \forall i = 1, 2, ..., N
\end{equation}
\end{definition}
\begin{definition}
For a submodular function $f$ and a score-based permutation $\textbf{x}$, we define a permutation $\sigma_{\textbf{x}}$ such that $\textbf{x}[\sigma_{\textbf{x}}(1)] \ge \textbf{x}[\sigma_{\textbf{x}}(2)]\ge ... \ge \textbf{x}[\sigma_{\textbf{x}}(N)]$, and a chain of sets $\phi = S_{0}^{\sigma_{\textbf{x}}} \subseteq S_{1}^{\sigma_{\textbf{x}}} \subseteq ... \subseteq S_{N}^{\sigma_{\textbf{x}}} = V$, with which the vector $h_{\sigma_{\textbf{x}}}^{f}$ is defined as (\ref{eq:hss}).
\begin{equation}
\label{eq:hss}
h_{\sigma_{\textbf{x}}}^{f}(\sigma_{\textbf{x}}(i)) = f(S^{\sigma_{\textbf{x}}}_{i}) - f(S^{\sigma_{\textbf{x}}}_{i-1}), \forall i = 1, 2, ..., N
\end{equation}
\end{definition}
\begin{definition}
Given a submodular function $f$ and its associated Lovasz extension $\hat{f}$, we define the LB divergence $d_{\hat{f}}(\textbf{x} || \sigma)$ for measuring the divergence of permutations between a score-based permutation $\textbf{x}$ and an order-based permutation $\sigma$. The LB divergence is shown as (\ref{eq:lvz}).
\begin{equation}
\label{eq:lvz}
\begin{split}
d_{\hat{f}}(\textbf{x} || \sigma) &= <\textbf{x}, h_{\sigma_{\textbf{x}}}^{f} - h_{\sigma}^{f}>	\\
\end{split}
\end{equation}
\noindent where both $h_{\sigma}^{f}$ and $h_{\sigma_{\textbf{x}}}^{f}$ are defined via definitions $1$ and $2$ respectively, $\phi$ denotes an empty set, and $V = \{1, 2, ..., N\}$ refers to the ground set.
\end{definition}

\subsection{The LB divergence for the NDCG loss function}
Next, we will study how to apply the LB divergence to obtain the Normalized Discounted Cumulative Gain (NDCG) loss function which is used in the unsupervised rank aggregation on score-based permutations. 

\begin{definition}
Given an order-based permutation $\sigma$ with candidates from a ground set $V = \{1, 2, ..., N\}$, and a set of relevance scores $\{r(1), r(2), ..., r(N)\}$ associated with $\sigma$, the NDCG score for $\sigma$ is defined as (\ref{eq:ndcg}), where $D(\cdot)$ denotes a discounted term and $Z$ refers to a normalization term as defined in (\ref{eq:norm}) where $\pi$ is the ground truth. Accordingly, the NDCG loss function $L(\sigma)$ is defined as (\ref{eq:loss}).
\begin{equation}
\label{eq:ndcg}
NDCG(\sigma) = \frac{1}{Z} \sum\limits_{i=1}^{N}r(\sigma(i))D(i)
\end{equation}
\begin{equation}
\label{eq:loss}
\begin{split}
L(\sigma) &= 1 - NDCG(\sigma)	 = \frac{1}{Z} \sum\limits_{i=1}^{N} r(\pi(i))D(i) - r(\sigma(i))D(i)  \\
\end{split}
\end{equation}
\begin{equation}
\label{eq:norm}
Z = \sum\limits_{i=1}^{N} r(\pi(i))D(i)
\end{equation}
\end{definition}

\begin{corollary}
For a submodular function $f(X) = g(|X|)$, the LB divergence $d_{\hat{f}}(\textbf{x} || \sigma)$ associated with $f$ is derived as (\ref{eq:ndcg_f}). 
\begin{equation}
\label{eq:ndcg_f}
d_{\hat{f}}(\textbf{x} || \sigma) = \sum\limits_{i=1}^{N} \textbf{x}(\sigma_{\textbf{x}}(i))\delta_{g}(i) - \sum\limits_{i=1}^{N}\textbf{x}(\sigma(i))\delta_{g}(i)
\end{equation}
where $\delta_{g}(i) = g(i) - g(i - 1)$, $g(\cdot)$ is a concave function, $|X|$ denotes a cardinality function, and $\sigma_{\textbf{x}}$ and $\sigma$ refer to the score-based and order-based permutations respectively. 
\end{corollary}

Obviously, it is found that $d_{\hat{f}}(\textbf{x}||\sigma) \propto L(\sigma)$, and the normalized $d_{\hat{f}}(\textbf{x} || \sigma)$ can be applied as a utility function for the NDCG loss measurement because the normalization term $Z$ is constant. In addition, the Lovasz Bregman divergence guarantees an upper bound for the NDCG loss function as shown in Proposition $6$.

\begin{proposition}
Given a score-based permutation $\textbf{x}$ and a concave function $g$, the LB divergence $d_{\hat{f}}(\textbf{x} || \sigma)$ defined as (\ref{eq:ndcg_f}) provides a constant upper bound to the NDCG loss function. Specifically, 
\begin{equation}
\label{eq:thm_ndcg}
\begin{split}
L(\sigma) &\le \frac{n}{Z}\cdot \epsilon \cdot (g(1) - g(|V|) + g(|V|-1)) \le \frac{\epsilon \cdot (g(1) - g(|V|) + g(|V|-1))}{\min\limits_{i}\textbf{x}(\sigma_{\textbf{x}}(i))\delta_{g}(i)}
\end{split}
\end{equation}
where $n$ is the number of permutation $\textbf{x}$, $\epsilon = \max_{i, j}|\textbf{x}(i) - \textbf{x}(j)|$, $Z = \sum_{i=1}^{n}\textbf{x}(\sigma_{\textbf{x}}(i))\delta_{g}(i)$ is a normalization term, and the upper bound for $L(\sigma)$ is independent of the permutation variable $\sigma$.
\end{proposition}

\begin{proof}
To obtain (\ref{eq:thm_ndcg}), we firstly show that, given a monotone submodular function $f$ and any permutation $\sigma$, there is an inequality \cite{rishabh_lbd} for $d_{\hat{f}}(\textbf{x} || \sigma)$ such that
\begin{equation}
\label{eq:pre_lbd}
d_{\hat{f}}(\textbf{x} || \sigma) \le \epsilon n\cdot (\max_{j}f(j) - \min\limits_{j}f(j| V\backslash \{j\}))
\end{equation}
where $\epsilon = \max_{i, j} |\textbf{x}(i) - \textbf{x}(j)|$, and $f(j | V\backslash \{j\}) = f(V) - f(V\backslash \{j\})$ is a marginal gain. Furthermore,
by setting $f(X) = g(|X|)$, we can obtain $\max_{j}f(j) = g(1)$, and by applying the submodular diminishing return property $\min_{j}f(j|V\backslash \{j\}) = f(V) - f(V\backslash \{j\}) = g(|V|) - g(|V|-1)$. Thus, by setting $L(\sigma) = \frac{1}{Z} d_{\hat{f}}(\textbf{x}||\sigma)$, and noting that $Z \ge n\cdot \min_{i}\sigma_{\textbf{x}}(i)\delta_{g}(i)$,  the proof for (\ref{eq:thm_ndcg}) is completed.  
\end{proof}

\section{The Unsupervised Learning Frameworks}
In this section, two unsupervised learning frameworks  based on the LB divergence and the associated algorithms are proposed.

\subsection{The Linear Structured Framework}
The first unsupervised learning framework based on the LB divergence is based on a linear structure. Suppose that there are $|Q|$ queries in total and $K$ score-based permutations $\{\textbf{x}^{q}_{1}, \textbf{x}^{q}_{2}, ..., \textbf{x}^{q}_{K}\}$ associated with the query $q\in Q$. $\forall q\in Q$, it is assumed there is a random variable $\pi^{q} \in [N]$ which subsumes $N!$ possible permutations, and the LB divergence $d_{\hat{f}}(\textbf{x}_{i}^{q} || \pi^{q})$ computes the divergence between $\textbf{x}^{q}_{i}$ and $\pi^{q}$. The problem is to study how to assign the weights $w_{i}$ to $d_{\hat{f}}(\textbf{x}_{i}^{q} || \pi^{q}), \forall i\in [1,...,K]$. The objective function of the linear structure is formulated as (\ref{eq:linear_obj}), where $E_{\pi^{q}}[\cdot]$ denotes an expectation term over a random variable $\pi^{q}$, and $\lambda$ denotes a regularization term.   
\begin{equation}
\label{eq:linear_obj}
\begin{split}
\min\limits_{\textbf{w}} \frac{1}{|Q|} &\sum\limits_{q\in Q} E_{\pi^{q}}[\sum\limits_{i=1}^{K} w_{i}d_{\hat{f}}(\textbf{x}_{i}^{q} || \pi^{q})] + \frac{\lambda}{2}\sum\limits_{i=1}^{K}w_{i}^{2} \\
&s.t.,  \sum\limits_{i=1}^{K} w_{i} = 1, \hspace{2mm} w_{i} \ge 0
\end{split}
\end{equation}

On the other hand, taking gradients with respect to $w_{i}$, we obtain $\nabla_{i}$ as (\ref{eq:gradient}). During the learning process, the Stochastic Gradient Descent (SGD) $\nabla_{i}^{sgd}$ is applied to update $w_{i}$ as follows: 
\begin{equation}
\label{eq:gradient}
\nabla_{i} = \frac{1}{|Q|} \sum\limits_{q\in Q} E_{\pi^{q}}[d_{\hat{f}}(x_{i}^{q} || \pi^{q})] + \lambda w_{i}
\end{equation}
\begin{equation}
\label{eq:gradient_sgd}
\nabla_{i}^{sgd} =  E_{\pi^{q}}[d_{\hat{f}}(x_{i}^{q} || \pi^{q})] + \lambda w_{i}
\end{equation}

The expectation term in (\ref{eq:gradient_sgd}) is approximated by the Metropolis-Hasting sampling method. More specifically, whether or not a sample is selected depends on the following steps (a), (b), and (c). 

(a). At time $t$, suppose the state lies in $\pi^{q}_{t}$. Given the weight $\textbf{w}$ and score-based permutations $\{\textbf{x}^{q}_{i}\}_{i=1}^{K}$, the probability of generating $\pi^{q}_{t}$ is based on the Mallows model as shown in (\ref{eq:prob}), where $Z(\textbf{w})$ is a normalization term that is independent of any permutation. 
\begin{equation}
\label{eq:prob}
P(\pi_{t}^{q}) = \frac{1}{Z(\textbf{w})}\exp(-\sum\limits_{i=1}^{K}w_{i} d_{\hat{f}}(\textbf{x}^{q}_{i} || \bar{\pi}_{t}^{q}))
\end{equation}
(b). Whether or not the next step goes to a new state $\bar{\pi}_{t}^{q}$ depends on the ratio $\alpha$ as defined in (\ref{eq:ratio}). 
\begin{equation}
\label{eq:ratio}
\alpha = \frac{P(\bar{\pi}^{q}_{t} | \textbf{w}, \{\textbf{x}_{i}\}_{i=1}^{K})}{P(\pi^{q}_{t} | \textbf{w}, \{\textbf{x}_{i}\}_{i=1}^{K})} = \frac{\exp(-\sum_{i=1}^{K}w_{i} d_{\hat{f}}(\textbf{x}^{q}_{i} || \bar{\pi}_{t}^{q}))}{\exp(-\sum_{i=1}^{K}w_{i} d_{\hat{f}}(\textbf{x}^{q}_{i} || \pi^{q}_{t}))}
\end{equation}
(c). If $\alpha > 0.9$, we accept $\pi^{q}_{t} \rightarrow \bar{\pi}^{q}_{t}$ with a probability at least $90\%$; otherwise, the state $\pi^{q}_{t}$ stays.  

\vspace{2mm}
As soon as $M$ samples are collected, the sampling process stops and the expectation term in (\ref{eq:gradient_sgd}) is estimated by taking an average of the $M$ samples according to (\ref{eq:gradient_mod}). 
\begin{equation}
\label{eq:gradient_mod}
E_{\pi^{q}}[d_{\hat{f}}(\textbf{x}_{i}^{q} || \pi^{q})] \approx \frac{1}{M} \sum\limits_{t=1}^{M} d_{\hat{f}}(\textbf{x}_{i}^{q} || \pi^{q}_{t})
\end{equation}  
Finally, given a learning rate $\mu$, the update for weights $\textbf{w}$ follows (\ref{eq:update_w}), which ensures the constraints for $\textbf{w}$.
\begin{equation}
\label{eq:update_w}
w_{i}^{(t+1)} = \frac{w_{i}^{(t)}\exp(-\mu \nabla^{sgd}_{i})}{\sum_{j=1}^{K}w_{j}^{(t)}\exp(-\mu \nabla^{sgd}_{j})}
\end{equation}
The above update goes through all $i=1,...,K$ and queries $q\in Q$. Several iterations are repeated until reaching convergence. 

\begin{proposition}
In the linear structured framework, given the number of candidates $N$ and the number of permutations $K$, the computational complexity in the training stage is $O(NK)$. 
\end{proposition}

The inference process is formulated as follows: given a test data $\{q, N, K, X\}$ where $q$ is a query, $N$ and $K$ have been defined as above, and $X = \{\textbf{x}^{q}_{1}, ..., \textbf{x}^{q}_{K}\}$ represents the score-based permutation associated with the query $q$, we estimate an optimal order-based permutation $\hat{\sigma}$ as defined in (\ref{eq:inference}).
\begin{equation}
\label{eq:inference}
\hat{\sigma} = \arg\min\limits_{\pi^{q}}\sum\limits_{i=1}^{K} w_{i}^{*} d_{\hat{f}}(\textbf{x}_{i}^{q} || \pi^{q})
\end{equation}
where $w^{*}$ refers to the weight vector trained in the learning stage. Generally, the problem in (\ref{eq:inference}) is an NP-hard combinatorial problem, but the LB divergence provides a close-form solution $\hat{\sigma} = \sigma_{\mu}$, where $\mu = \sum_{i=1}^{K}w_{i}^{*}\textbf{x}_{i}^{q}$. The complete inference algorithm is shown in Algorithm 1. Note that the order-based permutation $\hat{\sigma}$ is finally obtained by sorting the numeric values in $R_{X}$ and the permutation associated with $R_{X}$ is returned in a decreasing order.

\begin{algorithm}[tb]
   \caption{Inference Algorithm 1}
   \label{alg:inference}
\begin{algorithmic}
   \STATE {\bfseries 1.} Input: a test data $\{q, N, K, X\}$, and the trained weights $\textbf{w}^{*} = \{w_{1}^{*}, w_{2}^{*}, ... , w_{K}^{*}\}$.
   \STATE {\bfseries 2.} Compute $R_{X} = \sum_{i=1}^{K} w^{*}_{i}\textbf{x}_{i}$.
   \STATE {\bfseries 3.} Argsort $R_{X}$ in a decreasing order $\rightarrow \hat{\sigma}$.
   \STATE {\bfseries 4.} Output: $\hat{\sigma}$.
\end{algorithmic}
\end{algorithm}

\subsection{The Nested Structured Framework}
The linear structured framework involves several potential problems: the first problem is that the score-based permutations $\{\textbf{x}_{i}\}_{i=1}^{K}$ may not interact with each other, since one permutation might be partially redundant with another; the second problem lies in the fact that a permutation $\textbf{x}_{i}$ tends to become dominant over the rest. To overcome these problems, an additional hidden layer is utilized to construct a nested structured framework as shown in Figure \ref{fig:ns}.

\begin{figure}
\centerline{\epsfig{figure=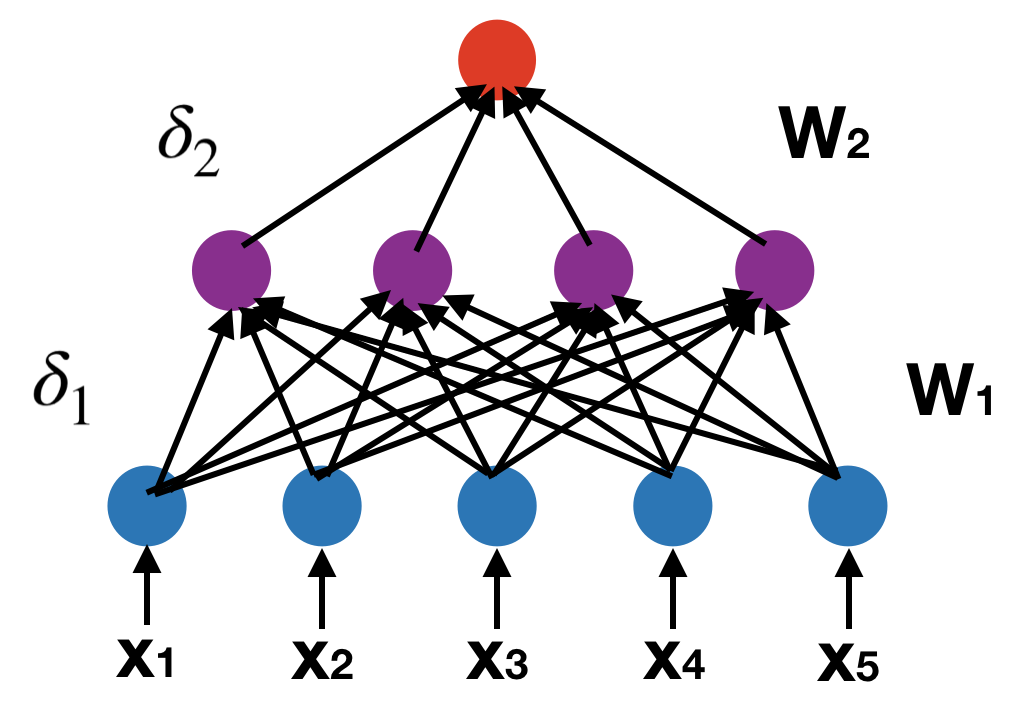,width=55mm}}
\caption{{\it An illustration of the nested structured framework.}}
\label{fig:ns}
\end{figure}

The objective function for the nested structured framework is formulated as (\ref{eq:nested}), where $K_{1}$ and $K_{2}$ separately represent the numbers of units in the input and hidden layers, $\textbf{W}_{1}\in R^{K_{2}*K_{1}}$ and $\textbf{W}_{2}\in R^{1 * K_{2}}$ denote weights for the bottom and upper layers respectively, and $\lambda_{1}$ and $\lambda_{2}$ refer to regularization terms.  By setting both $\Phi_{1}$ and $\Phi_{2}$ as increasing concave functions, the objective function becomes a concave function and thus it needs to be maximized with respect to $\textbf{W}_{1}$ and $\textbf{W}_{2}$.
\begin{equation}
\label{eq:nested}
\begin{split}
\max\limits_{\textbf{W}_{1}, \textbf{W}_{2}}&\frac{1}{|Q|}\sum\limits_{q\in Q} \Phi_{2}(\sum\limits_{i=1}^{K_{2}}\textbf{W}_{2}(i)\Phi_{1}(E_{\pi^{q}}[\sum\limits_{j=1}^{K_{1}}\textbf{W}_{1}(i, j)d_{\hat{f}}(\textbf{x}_{j}^{q}||\pi^{q})])) + \frac{\lambda_{1}}{2} ||\textbf{W}_{1}||_{F}^{2} + \frac{\lambda_{2}}{2} ||\textbf{W}_{2}||_{F}^{2}		\\
& s.t., \sum\limits_{j=1}^{K_{1}} \textbf{W}_{1}(i, j) = 1, \hspace{2mm} \sum\limits_{i=1}^{K_{2}}\textbf{W}_{2}(i) = 1,  \hspace{2mm} \textbf{W}_{1}(i, j) \ge 0, \hspace{2mm} \textbf{W}_{2}(i) \ge 0, \hspace{2mm} \forall i, j	\\
\end{split}
\end{equation}
The update for weights $\textbf{W}_{1}$ and $\textbf{W}_{2}$ follows a feed-forward manner that is similar to that employed for a standard Multiple Layer Perceptron (MLP) training. 

As to update for the weights of the bottom layer, the temporary variables $\delta_{1}^{(t)}(i)$ and $\nabla_{1}(i, j)$ need to be firstly computed by (\ref{eq:delta_1}) and (\ref{eq:nabla_1}) respectively, and then $\textbf{W}_{1}$ is updated by (\ref{eq:update_w1}). Note that the Metropolis-Hasting sampling method is applied in both (\ref{eq:delta_1}) and (\ref{eq:nabla_1}), where $M$ denotes the number of sampling permutations, and in (\ref{eq:update_w1}) $\mu$ refers to the learning rate. 
\begin{equation}
\label{eq:delta_1}
\begin{split}
\delta_{1}^{(t)}(i) &= E_{\pi^{q}} [\sum\limits_{j=1}^{K_{1}}\textbf{W}_{1}^{(t)}(i, j) d_{\hat{f}}(\textbf{x}_{j}^{q} || \pi^{q})] \approx \frac{1}{M} \sum\limits_{t=1}^{M}\sum\limits_{j=1}^{K_{1}}\textbf{W}_{1}^{(t)}(i, j) d_{\hat{f}}(\textbf{x}_{j}^{q}||\pi_{t}^{q})
\end{split}
\end{equation}
\begin{equation}
\label{eq:nabla_1}
\begin{split}
\nabla_{1}(i, j) &= (\Phi_{1}(\delta_{1}^{(t)}(i)))^{'} E_{\pi^{q}}[d_{\hat{f}}(\textbf{x}_{j}^{q}||\pi^{q})] + \lambda_{1} \textbf{W}_{1}^{(t)}(i, j)	\\
& \approx (\Phi_{1}(\delta_{1}^{(t)}(i)))^{'} \frac{1}{M} \sum\limits_{t=1}^{M}d_{\hat{f}}(\textbf{x}_{j}^{q} || \pi_{t}^{q})  + \lambda_{1} \textbf{W}^{(t)}_{1}(i, j)    \\
\end{split}
\end{equation}
\begin{equation}
\label{eq:update_w1}
\textbf{W}_{1}^{(t+1)}(i, j) = \frac{\textbf{W}_{1}^{(t)}(i, j) \exp(-\mu \nabla_{1}(i, j))}{\sum\limits_{j=1}^{K_{1}}\textbf{W}_{1}^{(t)}(i, j)\exp(-\mu \nabla_{1}(i, j))}, \hspace{2mm} \forall i, j 
\end{equation}
The update for $\textbf{W}_{2}$ starts when the update for $\textbf{W}_{1}$ is finished. The new $\delta_{2}^{(t)}$ and $\nabla_{2}(i)$ are separately derived via (\ref{eq:delta_2}) and (\ref{eq:nabla_2}) and they are based on the messages propagated from the bottom layer. Finally, the update for $\textbf{W}_{2}$ is computed by (\ref{eq:update_w2}). 

\begin{equation}
\label{eq:delta_2}
\delta_{2}^{(t)} = \sum\limits_{i=1}^{K_{2}}\textbf{W}_{2}^{(t)}(i) \Phi_{1}(\delta_{1}^{(t+1)}(i))
\end{equation}
\begin{equation}
\label{eq:nabla_2}
\nabla_2(i) = (\Phi_{2}(\delta_{2}^{(t)}))^{'} \Phi_{1}(\delta_{1}^{(t+1)}(i)) + \lambda_{2} \textbf{W}^{(t)}_{2}(i)
\end{equation}
\begin{equation}
\label{eq:update_w2}
\textbf{W}_{2}^{(t+1)}(i) = \frac{\textbf{W}_{2}^{(t)}(i)\exp(-\mu \nabla_{2}(i))}{\sum\limits_{j=1}^{K_{2}}\textbf{W}_{2}^{(t)}(j)\exp(-\mu \nabla_{2}(j))}, \hspace{2mm} \forall i
\end{equation}

\begin{proposition}
In the nested structured framework, given the numbers $K_{1}$ and $K_{2}$ for the input and hidden layers respectively, and the number of candidates $N$, the computational complexity for the entire training process is $O(NK_{1} + K_{1}K_{2})$.
\end{proposition}

The inference for the nested structured framework shares the same steps that the linear structured framework except that the step $2$ in Algorithm $1$ is replaced by (\ref{eq:nested_inference}).
\begin{equation}
\label{eq:nested_inference}
R_{X} = \Phi_{2}(\sum\limits_{i=1}^{K_{2}}\textbf{W}_{2}(i)\Phi_{1}(\sum\limits_{j=1}^{K_{1}}\textbf{W}_{1}(i, j)\textbf{x}_{j}^{q}))
\end{equation}

\section{Applications}
We tested our algorithms in four different applications and report the results in this section\footnote{The C++ codes for the experiments are uploaded to the Github website: https://github.com/uwjunqi/Subrank}. In addition, another unsupervised rank aggregation method based on ULARA is tested for the purpose of comparison with the LB divergence-based methods.  As introduced in Section 1, the applications include Information Retrieval, Combining Distributed Neural Networks, Influencers in Social Networks, and Distributed Automatic Speech Recognition (ASR). In the experiments, the Sigmoid function ($Sigmoid(x), x\ge 0$) was chosen for the discounted factor $D(x)$, $\Phi_{1}(\cdot)$ and $\Phi_{2}(\cdot)$. The learning rate $\mu$ is set to $0.1$ and all the regularization terms are fixed to $0.01$.

\subsection{Information Retrieval}
Given a query $q$, it is expected to find the top $C$ candidates $\{d^{(q)}_{1}, d^{(q)}_{2}, ..., d^{(q)}_{C}\}$ associated with potential relevance scores $\{r^{(q)}(1), r^{(q)}(2), ..., r^{(q)}(C)\}$. The NDCG score, which is defined in (\ref{eq:ndcg}) and lies in the interval between [0, 1], is a tool for evaluating the quality of the ranking $\sigma$ of the retrieved documents. A larger NDCG score represents a ranking result with more confidence. 

The experiments were conducted on the LETOR $4.0$ dataset, which is a package of benchmark data sets for research on Learning to Rank. The dataset contains standard features, relevance judgements, data partitioning, and several baselines. The LETOR $4.0$ dataset involves two query sets named MQ$2007$ and MQ$2008$ for short. There are about $1700$ queries in MQ$2007$ and about $800$ queries in MQ$2008$. For each query, there are no more than $40$ document candidates associated with the given relevant scores.

We applied our unsupervised rank aggregation methods on the dataset and compared the NDCG results with some learning to rank approaches. The learning to rank approaches include RankSVM \cite{RankSVM}, ListNet \cite{ListNet}, AdaRank \cite{AdaRank}, and RankBoost \cite{RankBoost}. As for the setup of the LB divergence-based methods, there were $K=46$ permutations and $N=40$ candidates associated with each query in total, and the number of units in the hidden layer for the nested structure is set to $10$.

\begin{figure}[htb]
\centerline{\epsfig{figure=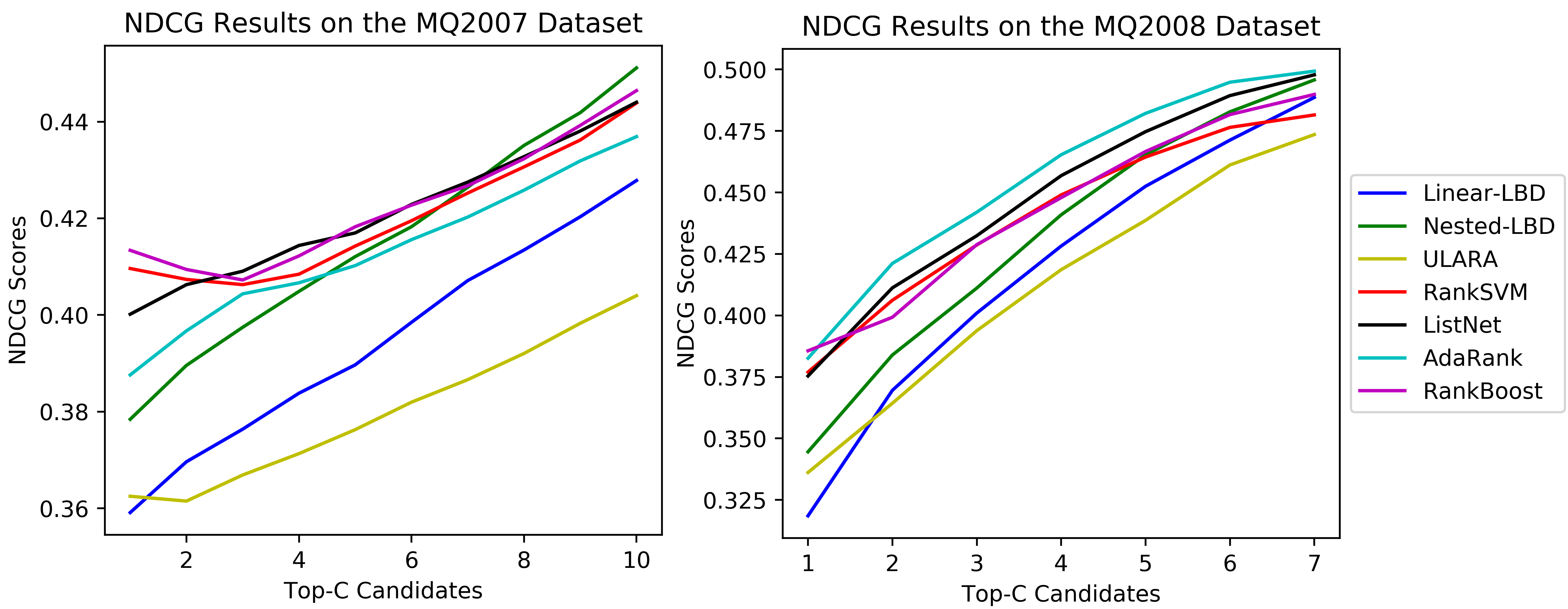,width=142mm}}
\caption{{\it Results on the LETOR dataset.}}
\label{fig:letor}
\end{figure}

The NDCG results on the MQ$2007$ and MQ$2008$ datasets are shown in Figure \ref{fig:letor}, where Linear-LBD and Nested-LBD represent linear and nested structured LB divergence-based methods for unsupervised rank aggregation respectively. The results show that the nested structured LB divergence-based method is comparable to the learning to rank methods and even obtains better results when more potential candidates are considered.

\subsection{Combining Distributed Neural Networks}
Big data requires deep learning architectures to be set up in a distributed way, for example in automatic speech recognition and machine translation. This study is about how to combine the hypothesis outputs from the distributed neural networks into one aggregated result. The result is expected to be as close to the ground truth as possible, which corresponds to a higher accuracy for prediction. If the outputs from the distributed neural networks are seen as score-based permutations, the task of combining distributed neural networks is taken as the unsupervised rank aggregation on score-based permutations.  
 
The combination of distributed neural networks was conducted via a digit recognition task on the MNIST dataset. The distributed neural networks were constructed as shown in Figure~\ref{fig:dnns}. The distributed system consisted of $6$ neural networks in total, where the first two were deep neural networks \cite{dnn}, the two in the middle were convolutional neural networks \cite{image_cnn}, and the two last were multi-layer perceptrons with the stacked Auto-encoder (SAE) initialization technique \cite{SDAE}. So, there are $6$ permutations in total in the distributed deep model architecture, and in each permutation there are $10$ scores for candidates from $0$ to $9$.   In addition, all the neural networks were set up in different architectures because of the changes in the number of units in the hidden layers and the size of the neuron units. 

\begin{figure}[t]
\centerline{\epsfig{figure=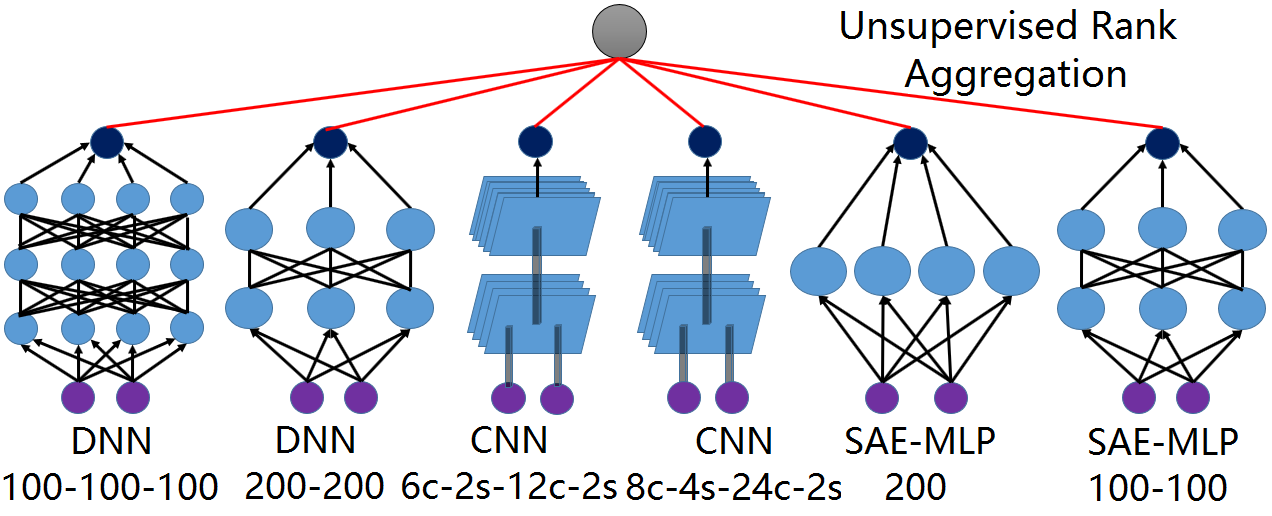,width=110mm}}
\caption{{\it The architecture of the distributed neural networks.}}
\label{fig:dnns}
\end{figure}
 
As shown in Figure $\ref{fig:dnns}$, an unsupervised rank aggregation layer was added to the top of the distributed system to combine the outputs of the neural networks. Table $1$ shows the detailed architecture configurations corresponding to the different neural networks and results based on the digit error rates (DERs). As to the configurations for the aggregated methods, the number of permutations $K$ is configured to $6$, the number of units of the hidden layer in the nested structured framework is set to $20$, and there is only one output corresponding to the final aggregated result. 

\begin{table}[htbp]
\begin{center}
\label{tab:res_sfm}
\begin{tabular}{|c||c|}
\hline
 Deep Model Architectures & DER  \\
\hline
DNN-`624-100-100-100-10'	&2.73\%	\\
\hline
DNN-`624-200-200-10'		&3.29\%	\\
\hline
CNN-`624-6c-2s-12c-2s-10'	&3.43\%	\\
\hline
CNN-`624-8c-4s-24c-2s-10'	&3.34\%	\\
\hline
SAE-MLP-`624-200-10'		&4.47\%	\\
\hline
SAE-MLP-`624-100-100-10'	&2.73\%	\\
\hline
\hline
Averaging			&\textbf{2.89}\%	\\
\hline
ULARA			&\textbf{2.52}\%	\\
\hline
Linear-LBD		&\textbf{2.70}\%	\\
\hline
Nested-LBD   		&\textbf{2.43}\%   	\\
\hline
\end{tabular}
\caption{Digit Error Rates on the MNIST dataset.}
\end{center}
\end{table}

Note that the experiments were conducted by the deep learning toolkit \cite{dnntool} and DERs in Table 1 are not the state-of-the-art results, so we just show that the unsupervised rank aggregation on score-based permutations can further lower DERs. Particularly, the method based on the nested structured LB divergence obtains the maximum gain, while ULARA performs even better than the method based on the linear structured LB divergence and the simple averaging method.
 
\subsection{Influencers in Social Networks}
People are pair-wisely connected in social networks where the pair-wise preference between two individuals is provided. The study of Influencers in Social Networks aims to predict the human judgement on who is more influential with high accuracy. We studied the unsupervised submodular rank aggregation methods approaching to the baseline results obtained by the supervised logistic regression. A Receiver Operating Characteristic (ROC) curve \cite{auc} is used to evaluate all the methods. The ROC score is in the interval $[0, 1]$, and a higher ROC value means a higher prediction accuracy.  

The data for the task are provided by the Kaggle competition task (Influencers in Social Networks) and comprise a standard, pair-wise preference learning task. For pair-wise preference data points A and B, a combined feature $X$ is pre-computed by (\ref{eq:pre_social}), where $X_{A}$ and $X_{B}$ are $11$ pre-computed, non-negative numeric features based on twitter activity, which include the volume of interactions and number of followers.
\begin{equation}
\label{eq:pre_social}
X = \log(1 + X_{A}) - \log(1 + X_{B})
\end{equation}

The binary label represents a human judgement about whom of two individuals is more influential. The goal of the task is to predict the human judgement on who is more influential with high accuracy. Specifically, for the unsupervised rank aggregation task, the purpose is to assign a likelihood to each candidate by aggregating $11$ features of the candidate. 

There are $5500$ labeled data points that are randomly split into a training set (which includes $4400$ data points) and a testing set with the rest. The baseline system is based on supervised logistic regression, which is used to compare to our unsupervised methods. As for the configuration of the LB divergence-based methods, the number of permutations $K$ is set to $11$, the number of units of the hidden layer in the nested structured framework is set to $40$, and the numbers of candidates $N$ for training and testing are set to $4400$ and $1100$ respectively. 

The results based on the ROC scores are shown in Table $2$. Note that all the results are an average of $10$ testing results based on different partitioned datasets. Although the results based on the unsupervised submodular rank aggregation methods are all below the baseline result, the approach based on the nested structured LB divergence is close to the baseline.  

\begin{table}[htbp]
\begin{center}
\label{tab:res_social}
\begin{tabular}{|c|c|c|c|}
\hline
Baseline &  Nested-LBD  & Linear-LBD & ULARA  \\
\hline
0.8631 & 0.8081 & 0.6777 & 0.7764    \\
\hline
\end{tabular}
\caption{The ROC scores on the dataset.}
\end{center}
\end{table}

\subsection{Distributed Automatic Speech Recognition}
The last application based on the unsupervised submodular rank aggregation refers to the distributed automatic speech recognition system. An illustration of the distributed ASR system is shown in Figure $\ref{fig:nMLPs}$. Compared to the DNN combination where each of the deep learning models has a different structure or configuration, all the deep learning models share the same initial setup including the types and the number of layers. A training dataset is partitioned into $8$ non-overlapping subsets by means of the robust submodular data partitioning methods~\cite{jun16}. Since each of the $8$ subsets is employed for training a particular DNN-based acoustic model, $8$ different DNN-based acoustic models are finally collected. In the evaluation stage, the test data are fed into all of the $8$ acoustic models, and all of the outputs from the distributed ASR system are expected to be aggregated into a combined result with a higher accuracy. 

Traditionally, the supervised Adaboost method \cite{jun16} is applied, but it is very expensive and difficult to obtain the ground truth in practice. Although approximated clustered triphone labels can be obtained via forced-alignment, the labels are not perfectly correct to supervise the Adaboost training. Thus, the unsupervised submodular rank aggregation on score-based permutations is attempted to replace the supervised Adaboost method. 

\begin{figure}[htb]
\centerline{\epsfig{figure=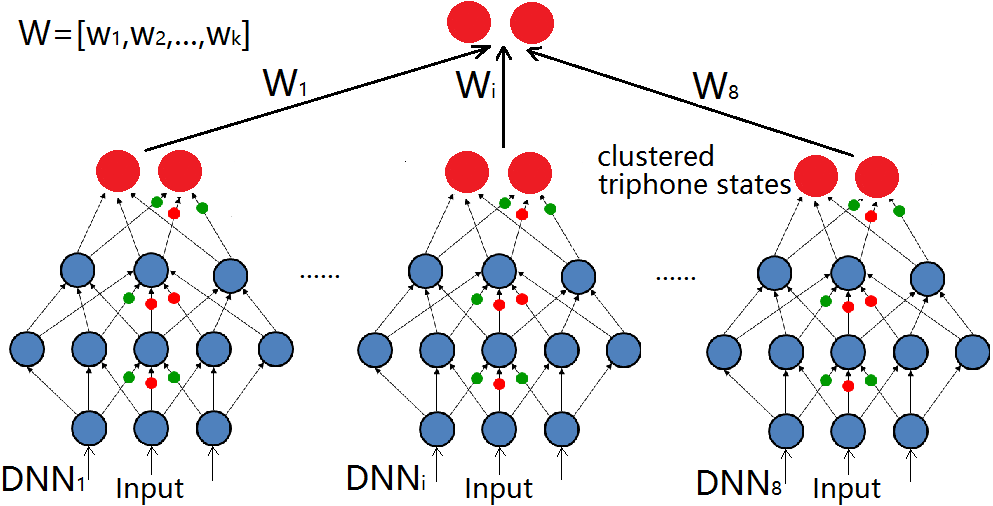,width=110mm}}
\caption{{\it The distributed ASR system.}}
\label{fig:nMLPs}
\end{figure}

Following the steps in \cite{jun16}, the submodular partitioning functions are composed according to the prior phonetic knowledge that a triphone corresponds to $8$ different biphones based on the phonetic knowledge including `Place of Articulation', `Production Manner', 'Voicedness' and `Miscellaneous'. For training each DNN-based acoustic model, the entire dataset is split into $8$ disjoint data subsets by formulating the problem as a robust submodular data partitioning problem as shown in (\ref{eq:partition}). 
\begin{equation}
\label{eq:partition}
\begin{split}
&\max\limits_{\pi\in \Pi}\min\limits_{i} f_{i}(A_{i}^{\pi})  \\
s.t., \cup_{i}A_{i}^{\pi} &= V, \hspace{2mm} A_{i}^{\pi} \cap A_{j}^{\pi} = \Phi, \forall i, j	\\
\end{split}
\end{equation} 

\noindent where $\pi=(A_{1}^{\pi}, A_{2}^{\pi}, ..., A_{m}^{\pi})$ is a partition of a finite set $V$, $\Pi$ denotes the sets corresponding to all possible partitions of $V$, and $A_{i}^{\pi}$ represents a partitioned data subset. Note that $\forall i, j$, the intersection of any two sets $A_{i}^{\pi}$, $A_{j}^{\pi}$ is empty, while the union of them covers the entire dataset. In addition, $\{f_{i}\}_{i=1}^{8}$ refers to $8$ heterogeneous submodular functions which are composed by mapping from $1$ triphone to $8$ biphones. The Minorization-Maximization (MMAX) algorithm is applied to obtain the approximated solutions to the problem. 

The experiments were conducted on the TIMIT database. The training data consist of $3696$ utterances in total. The development and test data are composed of $200$ and $1200$ utterances, respectively. Data preprocessing included extracting $39$-dimensional Mel Frequency Cepstral Coefficient (MFCC) features that correspond to $25.6$ms speech signals. In addition, mean and variance speaker normalization were also applied. The acoustic models were initialized as clustered triphones modeled by $3$-state left-to-right hidden Markov models (HMMs). The state emission of the HMM was modeled by the Gaussian mixture model (GMM). The DNN targets consisted of $3664$ clustered triphone states. A 3-gram language model was used for decoding. 

The $8$ subsets of data partitioned by the submodular functions were used for training $8$ DNNs in parallel. The units at the input layer of each DNN correspond to a long-context feature vector that was generated by concatenating $11$ consecutive frames of the primary MFCC feature followed by a discrete cosine transformation (DCT)~\cite{auditory}. Thus, the dimension of the initial long-context feature was $429$ which was reduced to $361$ after the DCT~\cite{smbf}. In addition, there were $4$ hidden layers with a setup of $1024$-$1024$-$1024$-$1024$ for each DNN. The parameters of the hidden layers were initialized via Restricted Boltzmann Machine pre-training, and then fine-tuning by the MLP Back-propagation algorithm~\cite{gammatone}. Besides, the feature-based maximum likelihood linear regression was applied for the DNN speaker adaptation~\cite{gales98}. 

When the training of all the DNN-based acoustic models was done, the final posteriors of the clustered triphones associated with the training data should be separately obtained from each of the DNN-based acoustic models. Those posteriors were taken as permutation data for training the unsupervised rank aggregation models. In the testing stage, the posteriors collected from the $8$ DNN-based acoustic models are combined to one permutation that is expected to be as close to the ground truth as possible. 

For the configuration of the two unsupervised submodular rank aggregation formulations, the number of permutations $K$ is set to $8$ and the dimension of a permutation is configured as $3664$ which matches the number of clustered triphones. Besides, the number of units of the hidden layer in the nested structured framework is set to $20$, and there is only one output corresponding to the final aggregated permutation. 

Table $3$ shows the ASR decoding results from each of the DNN-based acoustic models, and the Table $4$ presents the combined ones based on the different aggregation methods. The results suggest that the unsupervised submodular rank aggregation method based on the nested structured formulation achieves better result than the baseline system based on the Adaboost method, whereas the others are worse than the baseline. The marginal gain by the nested structured formulation arises from the potential bias caused by forced-alignment. 

\begin{table}[htbp]
\begin{center}
\label{tab:res}
\begin{tabular}{|c||c|c|c|c|c|c|c|c|}
\hline
Methods  &DNN1 &DNN2  &DNN3  &DNN4   &DNN5   &DNN6   &DNN7   &DNN8 \\
\hline
PER  &20.7 & 20.1 & 20.2 & 20.6   & 20.4 & 20.8 & 20.1  & 20.3		 \\
\hline
\end{tabular}
\caption{Phone Error Rate (PER) (\%).}
\end{center}
\end{table}

\begin{table}[htbp]
\begin{center}
\label{tab:res}
\begin{tabular}{|c||c|c|c|c|c|}
\hline
Methods &Averaging &Adaboost  &ULARA    &Linear-LBD  &Nested-LBD \\
\hline
PER  &20.1 & 18.3 & 19.5  &18.7  & \textbf{18.1}	 \\
\hline
\end{tabular}
\caption{Phone Error Rate (PER) (\%).}
\end{center}
\end{table}

\section{Conclusions and Future Work}
This study focuses on several algorithms for unsupervised submodular rank aggregation on score-based permutations based on the LB divergence in both linear and nested structured frameworks. Their use in Information Retrieval, Combining Distributed Neural Networks, Influencers on Social Networks, and Distributed Automatic Speech Recognition tasks suggest that the nested structured LB divergence can obtain significantly more gains. However, the gains obtained with respect to the other approaches are lower to varying degrees. In addition, our methods can be scalable to large-scale datasets because of their low computational complexity. 

Future work will study how to generalize the nested structure to a deeper structure with more hidden layers. Although the convexity of the objective function with a deep structure can be maintained, the use of the message-passing method to deeper layers cannot obtain a satisfying result in the applications. Therefore, a better unsupervised learning approach for training an LB divergence objective function with a deeper structure formulation is necessary. 

\small

\bibliography{subrank}
\bibliographystyle{nips_2017}
\normalsize
\newpage
\section{Appendix}

\subsection{Proof for the Corollary 5}
For a submodular function $f(X) = g(|X|)$, the Lovasz extension $\hat{f}$ associated with $f$ is 
\begin{equation}
\hat{f}(x) = \sum\limits_{i=1}^{K} \textbf{x}(\sigma_{\textbf{x}}(i))\delta_{g}(i) = <\textbf{x}, h_{\sigma_{\textbf{x}}}^{f}>
\end{equation}
where $\delta_{g}(i) = g(i) - g(i - 1)$. Then applying the Lovasz Bregman divergence (8), we have
\begin{equation}
d_{\hat{f}}(\textbf{x} || \textbf{y}) = <\textbf{x}, h_{\sigma_{\textbf{x}}}^{f} - h_{\sigma_{\textbf{y}}}^{f}> = \sum\limits_{i=1}^{K}\textbf{x}(\sigma_{\textbf{x}}(i))\delta_{g}(i) - \textbf{x}(\sigma_{\textbf{y}}(i))\delta_{g}(i)
\end{equation}

\subsection{Proof for equation (9)}
\begin{theorem}
Given a monotone submodular function $f$ and any permutation $\sigma$, there is an inequality for $d_{\hat{f}}(\textbf{x}||\sigma)$ such that 
the equation (12) holds. That is, $d_{\hat{f}}(\textbf{x} || \sigma) \le \epsilon n\cdot (\max_{j}f(j) - \min_{j}f(j| V\backslash \{j\}))$, where $\epsilon = \max_{i,j} |x_{i} - x_{j}|$ and $f(j|A) = f(A \cup \{j\}) - f(A)$. 
\end{theorem}

\begin{proof}
Decompose $x = \min_{j} x_{j}\textbf{1} + r$, where $r_{i} = x_{i} - \min_{j}x_{j}$. Notice that $|r_{i}|\le \epsilon$. Moreover, $\sigma_{x} = \sigma_{r}$ and hence $d_{\hat{f}}(\textbf{x}||\sigma) = d_{\hat{f}}(\min_{j}x_{j}\textbf{1} || \sigma) + d_{\hat{f}}(r || \sigma) = d_{\hat{f}}(r || \sigma)$ since $<\textbf{1}, h_{\sigma_{r}}^{f} - h_{\sigma}^{f}> = f(V) - f(V) = 0$. Now, $d_{\hat{f}}(r||\sigma) = <r, h_{\sigma_{r}}^{f} - h_{\sigma}^{f}> \le ||r||_{2}||h_{\sigma_{r}}^{f} - h_{\sigma}^{f}||_{2}$. Finally, note that $||r||_{2} \le \epsilon \sqrt{n}$ and $||h_{\sigma_{r}}^{f}-h_{\sigma}^{f}||_{2} \le \sqrt{n}(\max_{j}f(j) - \min_{j}f(j|V\backslash \{j\}))$ and combining these, we get the result.
\end{proof}

\subsection{The NDCG results in Figure $1$}

The exact NDCG values associated with the plots in Figure $1$ are listed in Tables $3$ and $4$ respectively. 

\begin{table}[htbp]
\begin{center}
\label{tab:res_values_mq2007}
\begin{tabular}{|c||c|c|c|c|c|c|c|c|c|c|}
\hline
Methods	 & Top-1 & Top-2 & Top-3 & Top-4 & Top-5 & Top-6 & Top-7 & Top-8 & Top-9 & Top-10  \\
\hline
Linear-LBD	&0.3591	&0.3696	&0.3764	&0.3838	&0.3897	&0.3985	&0.4071	&0.4134	&0.4203	&0.4278  \\
\hline
Nested-LBD	&0.3784	&0.3896	&0.3975	&0.4049	&0.4121	&0.4183	&0.4264	&0.4351	&0.4419	&0.4511  \\
\hline
ULARA		&0.3625	&0.3615	&0.3669	&0.3713	&0.3763	&0.3820	&0.3866	&0.3920	&0.3983	&0.4040  \\
\hline
\hline
RankSVM		&0.4096	&0.4074	&0.4063	&0.4084	&0.4143	&0.4195	&0.4252	&0.4306	&0.4362	&0.4439  \\
\hline
ListNet		&0.4002	&0.4063	&0.4091	&0.4144	&0.4170	&0.4229	&0.4275	&0.4328	&0.4381	&0.4440  \\
\hline
AdaRank		&0.3876	&0.3967	&0.4044	&0.4067	&0.4102	&0.4156	&0.4203	&0.4258	&0.4319	&0.4369  \\
\hline
RankBoost	&0.4134	&0.4094	&0.4072	&0.4122	&0.4183	&0.4227	&0.4267	&0.4323	&0.4392	&0.4464  \\
\hline
\end{tabular}
\caption{The NDCG results on MQ2007 dataset.}
\end{center}
\end{table}

\begin{table}[htbp]
\begin{center}
\label{tab:res_values_mq2007}
\begin{tabular}{|c||c|c|c|c|c|c|c|}
\hline
Methods	 & Top-1 & Top-2 & Top-3 & Top-4 & Top-5 & Top-6 & Top-7  \\
\hline
Linear-LBD	&0.3185	&0.3696	&0.4010	&0.4282	&0.4525	&0.4713	&0.4886  \\
\hline
Nested-LBD	&0.3446	&0.3839	&0.4111	&0.4409	&0.4653	&0.4838	&0.4957  \\
\hline
ULARA		&0.3362	&0.3643	&0.3938	&0.4186	&0.4386	&0.4612	&0.4735  \\
\hline
\hline
RankSVM		&0.3770	&0.4062	&0.4286	&0.4490	&0.4644	&0.4764	&0.4815  \\
\hline
ListNet		&0.3754	&0.4112	&0.4324	&0.4568	&0.4747	&0.4894	&0.4978  \\
\hline
AdaRank		&0.3826	&0.4211	&0.4420	&0.4653	&0.4821	&0.4948	&0.4993  \\
\hline
RankBoost	&0.3856	&0.3993	&0.4288	&0.4479	&0.4666	&0.4816	&0.4898  \\
\hline
\end{tabular}
\caption{The NDCG results on MQ2008 dataset.}
\end{center}
\end{table}

\end{document}